\documentclass[twoside]{article}

\usepackage{natbib}
\usepackage{amsmath,amssymb,amsthm,graphicx,booktabs,xcolor}
\usepackage[hidelinks]{hyperref}
\usepackage{algorithm,algorithmic}
\usepackage[preprint]{aistats2027}
\newtheorem{theorem}{Theorem}
\newtheorem{proposition}{Proposition}
\newtheorem{corollary}{Corollary}
\newtheorem{lemma}{Lemma}
\newtheorem{assumption}{Assumption}
\newtheorem{definition}{Definition}
\newcommand{\PR}{\mathbb{P}}
\newcommand{\E}{\mathbb{E}}
\newcommand{\ind}{\mathbf{1}}
\newcommand{\FAR}{\mathrm{FAR}}
\newcommand{\TAR}{\mathrm{TAR}}
\newcommand{\cP}{\mathcal{P}}

\newcommand{\AdvCPMax}{100}
\newcommand{\AdvCatMax}{2.48}
\newcommand{\AdvCatMin}{0.33}

\newcommand{\AdvFixedMax}{1.44}
\newcommand{\AdvFixedMin}{0.50}
\newcommand{\AdvNaiveMax}{3.37}
\newcommand{\BonfZ}{4.03}
\newcommand{\BootFARRatioMaxUnder}{37992}
\newcommand{\BootPACViolMaxWell}{20}
\newcommand{\BootPACViolMedMid}{21}
\newcommand{\BootPACViolMedUnder}{99}
\newcommand{\BootPACViolMedWell}{6}
\newcommand{\BootTARvsCoveMed}{90\%}

\newcommand{\CohortGainMax}{3.8}

\newcommand{\CohortGainMedShown}{2.1}
\newcommand{\CoveBeatsOracleFrac}{37}
\newcommand{\CoveBeatsOracleMaxPct}{21}

\newcommand{\CoveDrawViolMed}{46}

\newcommand{\CoveMaxAbsZ}{2.92}

\newcommand{\CoveMedRatio}{1.001}
\newcommand{\CoveMinAlpha}{10^{-7}}
\newcommand{\CoveMinRatioPct}{83}
\newcommand{\CoveNoneReject}{none}

\newcommand{\CoveValidityCells}{903}
\newcommand{\CslsMinRatio}{0.37}
\newcommand{\CslsWorseFrac}{24}
\newcommand{\DevGPUHours}{4.6}
\newcommand{\DevUntimedRecords}{51}
\newcommand{\DisjCPCombos}{622}
\newcommand{\DisjCPNontrivCells}{143}
\newcommand{\DisjCPTARvsCoveMed}{87\%}
\newcommand{\DisjCPViolMax}{8}
\newcommand{\EvalGPUHours}{4.3}
\newcommand{\FeatureGPUHours}{0.6}
\newcommand{\HubLoadThreeAcc}{65\%}
\newcommand{\HubLoadZeroAcc}{96\%}

\newcommand{\KappaMaxFour}{76}
\newcommand{\KappaMaxTwo}{3}

\newcommand{\KappaMinFour}{4.9}
\newcommand{\KappaMinTwo}{0.4}
\newcommand{\MainTARRatio}{100\%}
\newcommand{\MainTARRatioQlo}{92\%}
\newcommand{\NCells}{149}
\newcommand{\NDatasets}{6}
\newcommand{\NEncoders}{8}
\newcommand{\NGenPerDraw}{20{,}000}
\newcommand{\NImpPerDraw}{500{,}000}
\newcommand{\NPowerCells}{374}
\newcommand{\NReps}{500}
\newcommand{\NRepsOther}{200}

\newcommand{\NaiveCPMedLargeM}{42}

\newcommand{\NaiveCPMedSmallM}{8}
\newcommand{\NaiveCPWorstPct}{55}
\newcommand{\NaiveFracOverResolved}{36}

\newcommand{\NaiveWorstRatioResolved}{2.87}
\newcommand{\OracleAtLevel}{at COVE's exact level}
\newcommand{\OracleUnreliableDropped}{35}
\newcommand{\PairsIidFARMax}{1.61}
\newcommand{\PairsIidRejectPct}{5.3}
\newcommand{\PairsIidTARRatioMed}{100\%}
\newcommand{\PeakMemGB}{9.9}

\newcommand{\RhoAlphaMaxFour}{320}

\newcommand{\RhoAlphaMinFour}{44}

\newcommand{\RhoMaxFour}{0.032}

\newcommand{\RhoMinFour}{0.004}

\newcommand{\SnormWorseFrac}{26}
\newcommand{\TARRatioB}{100\%}
\newcommand{\TARRatioC}{100\%}
\newcommand{\UsPerTrialTenThousand}{3.3}
\newcommand{\UsPerTrialThousand}{0.7}

\begin{document}



\twocolumn[

\aistatstitle{Identities, Not Pairs: Distribution-Free False-Accept\\Control for Embedding-Based Verification}

\aistatsauthor{ Salim Khazem }

\aistatsaddress{ Talan Research Center } ]

\begin{abstract}
Embedding-based verification often calibrates a false-accept threshold from all impostor pairs among $m$ identities, although these pairs share identities. We characterize a worst-case limit: when hub identities score highly against a fraction $\rho$ of the population, a calibration-only threshold with marginal false-accept rate (FAR) at most $\alpha$ under every identity-exchangeable law
reaches the hub score on hub-free calibration sets with probability at most $(m+2)\alpha/(2\rho)$. Uniformly accepting near that score at levels well below $\rho/m$ therefore requires trial information. We propose COVE, which ranks the trial pair among all $N=\binom{m+2}{2}$ pairs formed by augmenting the calibration units with the probe and reference. Under identity exchangeability,
its marginal FAR is at most $\lfloor\alpha N\rfloor/N$, with first possible acceptance at $1/N=O(1/m^2)$. The rule also retains this guarantee with a fixed symmetric cohort-normalized score using an independent cohort. On 6 benchmarks and 8 frozen encoders, realised FAR was consistent with the exact bound. Where both valid calibration-only baselines accepted nothing, COVE retained a median $100\%$ of the sampled oracle's true-accept rate over reliable comparisons. In contrast, an empirical quantile over dependent pairs reached 2.87$\,\alpha$ FAR, and a binomial upper confidence bound over pairs missed its nominal 95\% conditional-FAR target in up to 55\% of calibration draws. Symmetric CSLS raised COVE's true-accept rate by up to 4.3-fold in the tested settings. The code is available in \url{https://github.com/salimkhazem/Cove.git}
\end{abstract}

\section{Introduction}
\label{sec:introduction}
A large class of vision systems decides whether two images show the \emph{same identity} by thresholding the similarity of frozen embeddings: product and landmark retrieval, animal re-identification, near-duplicate removal in dataset curation, and open-set matching with foundation encoders such as DINOv2 or CLIP, whether deployed directly or tailored via parameter-efficient adaptation, visual prompting, or distillation \citep{khazem2026multi,khazem2026topolora,khazem2026safe} face and person verification rest on the same decision, but we do not study them. The operator of such a system specifies a tolerated \emph{false-accept rate} (FAR) the probability of
declaring a match between two different identities and it is typically small, $10^{-4}$ to $10^{-6}$. The threshold is set on a labelled calibration set, and the standard recipe takes a quantile of the similarity over all \emph{impostor pairs}, the $m(m-1)/2$ pairs of images from different calibration identities, often with a binomial confidence bound that treats the pairs as independent draws. Those pairs are not independent: every identity appears in $m-1$ of them, and an identity that resembles many others, a ``wolf'' in the biometric zoo \cite{doddington1998sheep}, a ``hub'' in high-dimensional retrieval \cite{radovanovic2010hubs}, inflates many pair scores at once. This paper asks what can be guaranteed, without distributional assumptions, from $m$ calibration identities, and answers with a separation.

Our contributions are threefold. (i) We prove a sharp limitation of calibration-only FAR control for pairwise similarity problems: under identity exchangeability, hub identities can force any valid calibration-only threshold to operate at resolution of order $\rho/m$, yielding an order-$\rho m$ gap relative to the transductive scale; the bound is tight in a canonical setting (Theorem~\ref{thm:lower}). (ii) We introduce COVE, a transductive conformal procedure that ranks the trial within the augmented pairwise array. COVE achieves distribution-free marginal FAR control at level $\alpha$, rank resolution of order $1/m^2$, and retains validity with symmetric scores learned from an independent cohort (Theorem~\ref{thm:valid}, Corollary~\ref{cor:cohort}, Proposition~\ref{prop:power}. (iii) Across \NDatasets\ benchmarks and \NEncoders\ frozen encoders, COVE consistently controls FAR, while standard pair-based calibration degrades in the presence of hubness. In regimes where valid calibration-only rules become trivial, COVE retains a median \MainTARRatio\ of the oracle true-accept rate.

\section{Related work}\label{sec:related}

\paragraph{Conformal inference.} Exchangeable conformal methods \citep{vovk2005algorithmic,lei2018distribution,angelopoulos2023gentle} include outlier and open-set tests \citep{bates2023testing,xie2025openset}. Pair-indexed ranks have appeared in network and dyadic inference
\citep{lunde2023network,lunde2026dyadic,marandon2024conformal,zhao2024conformalized,fermanian2025transductive}; hierarchical and adaptive transductive settings are also studied \citep{dunn2023distribution,gazin2023transductive}. We use that known rank argument for verification, and study its separation from calibration-only thresholds, per-probe behavior and symmetric cohort scores. Related PAC and conditional-coverage limits \citep{vovk2012conditional,park2020pac,barber2021limits} and transductive efficiency bounds \citep{behboodi2025fundamental} address different targets.

\paragraph{Biometric error rates.} Impostor pairs share identities; subset bootstrap and correlated-binomial methods quantify the resulting uncertainty \citep{bolle2004subsets,schuckers2010computational,dass2006validating,fallahi2025reliability,conti2024assessing,fogliato2023confidence}. Biometric wolves and retrieval hubs are documented \citep{doddington1998sheep,une2008wolf,durbet2023untargeted}, and synthetic
identities have been studied for threshold calibration \citep{apila2026synthetic}. Our target is a finite-sample decision rule rather than an error-rate interval.

\paragraph{Score normalization.} Speaker and face verification use cohort normalization  \citep{auckenthaler2000score,choi2025trainable,li2026unified}; retrieval uses CSLS and query bank variants \citep{radovanovic2010hubs,conneau2018word,bogolin2022qbnorm,pan2025dualbank}. These normalizations alone do not establish the distribution-free marginal FAR bound; a fixed symmetric form within COVE inherits it (Corollary~\ref{cor:cohort}).

\section{Setting}\label{sec:setting}

A \emph{unit} is one image of one identity, represented by a frozen embedding. A \emph{pair score} $g(u,v)$ is any measurable \emph{symmetric} function of two units; it may also depend on auxiliary data independent of all units, such as an impostor cohort (Section~\ref{sec:cohort}). The canonical choice is cosine similarity. A verification \emph{trial} is a pair (probe, reference); it is an \emph{impostor} trial if the two units belong to different identities, and a \emph{genuine} trial otherwise.

The calibration set contains units $U_1,\dots,U_m$ of $m$ distinct identities, one unit per identity.\footnote{With several images per identity we draw one at random. When hubness is a property of the identity, Theorem~\ref{thm:lower} holds for any number of images per calibration identity, so extra images cannot move what a calibration-only threshold can guarantee below order $\rho/m$.} A trial contributes $U_{m+1}$ (probe) and $U_{m+2}$ (reference). Write $\cP=\{\{i,j\}:1\le i<j\le m+2\}$ for the $N=\binom{m+2}{2}$ unordered pairs of the augmented set and $S_e=g(U_i,U_j)$ for $e=\{i,j\}$; the trial pair is $e^\star=\{m+1,m+2\}$.

\begin{assumption}[identity exchangeability]\label{ass:exch}
Under the impostor null $H_0$, the trial's two units belong to two distinct identities that do not occur in calibration, and the vector $(U_1,\dots,U_{m+2})$ is exchangeable. 
\end{assumption}

Assumption~\ref{ass:exch} holds when identities are sampled from a common population (without replacement, or i.i.d.\ from a diffuse law) and each unit is an independent image of its identity. It is the identity-level analogue of the exchangeability that split conformal prediction assumes for samples \citep{vovk2005algorithmic,angelopoulos2023gentle}; it does \emph{not} require pairs to be independent. It requires probe and reference images to be captured the same way as calibration images (an enrolled template taken under a different protocol violates it), and it concerns zero-effort impostors (unrelated identities), not targeted look-alikes. It fails when calibration and deployment identities come from different populations, which we test in Section~\ref{sec:experiments} (Q6).

A decision rule maps the calibration units and the trial to accept/reject. Its false-accept rate is $\FAR=\PR_{H_0}(\text{accept})$, the probability over calibration and trial the marginal guarantee of conformal inference. Its true-accept rate $\TAR$ is the acceptance probability on genuine trials. A rule is inductive (calibration-only) if it accepts if $S_{e^\star}\ge\hat\tau$ (or $S_{e^\star}>\hat\tau$) for a threshold $\hat\tau$ computed from the calibration units alone (and independent randomization) for instance from their pair scores; this includes every quantile-based recipe. A rule is trivial at level $\alpha$ if it rejects every trial whatever the calibration set, and its resolution is the smallest level at which it is not.

\section{What calibration-only thresholds can guarantee}\label{sec:lower}

The recipe ``take the $(1-\alpha)$-quantile of the $\binom m2$ impostor pair scores'' suggests a resolution of order $1/m^2$. A hub construction shows that no calibration-only threshold can \emph{guarantee} it without assumptions.

\begin{definition}[hubs]\label{def:hub}
Fix a unit law $Q$ and a level $\bar c$. A diffuse law $H$ of units is a \emph{$\rho$-hub at level $\bar c$ for $Q$} if, for independent $V,V'\sim H$ and $U\sim Q$, both $\PR(g(V,U)\ge\bar c)\ge\rho$ and $\PR(g(V,V')\ge\bar c)\ge\rho$. Let $\rho_Q(\bar c)$ be the supremum of such $\rho$; it is non-increasing in $\bar c$.
\end{definition}

For unquantized cosine similarity on the full unit sphere in dimension $d\ge2$, $\rho_Q(\bar c)$ lies between the largest strict and non-strict spherical-cap masses of $Q$ at $\bar c$ (Appendix~\ref{app:lower}, Remark (e)). We model identities as drawn from a diffuse population, so that $Q$ is the image law of a random identity and a hub unit is an identity of
its own. Write $\mathcal L_{Q,H}$ for the law of $m+2$ units in which one unit, at a uniformly random position, is drawn from $H$ and the others i.i.d.\ from $Q$; it is exchangeable, so it satisfies Assumption~\ref{ass:exch}.

\begin{theorem}[calibration-only thresholds]\label{thm:lower}
Let $\hat\tau=T(U_1,\dots,U_m,\xi)$ be any threshold computed from the calibration units and independent randomization $\xi$, used as ``accept iff $S_{e^\star}\ge\hat\tau$''. (a) If $\PR(\text{accept})\le\alpha$ under every $\mathcal L_{Q,H}$ with $H$ a $\rho$-hub for $Q$ at level $\bar c$, then for every $c\le\bar c$ ($c<\bar c$ if the rule accepts iff $S_{e^\star}>\hat\tau$),
\[
\PR_Q(\hat\tau\le c)\;\le\;\frac{(m+2)\,\alpha}{2\rho}.
\]
(b) If $0<\alpha<\rho$ and instead the calibration-conditional FAR, $\FAR(\hat\tau\mid U_{1:m},\xi)$, must exceed $\alpha$ with probability at most $\delta$ under $Q^{\otimes m}$ contaminated
i.i.d.\ by $H$ at every rate $\pi\in(0,1)$, then $\PR_Q(\hat\tau\le c)\le\delta\,(1-\alpha/\rho)^{-m/2}$, which is $\delta\exp\{m\alpha/(2\rho)\,(1+O(\alpha/\rho))\}$.
\end{theorem}

\emph{Proof idea.} (a) Under $\mathcal L_{Q,H}$ the hub is a trial unit with probability $2/(m+2)$; the calibration units are then i.i.d.\ $Q$, so $\hat\tau\le c$ with the same probability as under $Q$, independently of the trial, and the trial scores $\ge\bar c\ge\hat\tau$ with probability at least $\rho$. (b) Contamination just above rate $1-\sqrt{1-\alpha/\rho}$ makes every threshold $\le c$ exceed $\alpha$, while leaving the calibration set hub-free with probability $(1-\pi)^m$
(Appendix~\ref{app:lower}). \hfill$\square$

Theorem~\ref{thm:lower} is a minimax statement: it bounds what a calibration-only rule can \emph{guarantee} over all laws with hubs of strength $\rho$, not what happens for a particular encoder. With $\rho$ of order one, a rule that is valid under every identity-exchangeable law has a threshold at or below $\bar c$ so accepts trials scoring $\bar c$ on hub-free calibration sets with probability at most $(m+2)\alpha/(2\rho)$, and a PAC-valid one with probability at most
$\delta e^{m\alpha(1+O(\alpha/\rho))/(2\rho)}$, close to $\delta$ for $\alpha\ll2\rho/m$ and below one only for $\alpha<2\rho\log(1/\delta)/m$. Uniformly lowering the threshold to the hub level requires $\alpha$ of order
$\rho_Q(\bar c)/m$, rather than $1/m^2$, even with more images per calibration identity when hubness is a property of the identity (Appendix~\ref{app:lower}). For $\rho=1$ and even $m$ the bound is tight: split conformal prediction on
the $\lfloor m/2\rfloor$ disjoint calibration pairs, exchangeable with the trial pair, is valid and first able to accept when $\alpha\ge1/(\lfloor m/2\rfloor+1)$, where the bound reaches one; a randomized version attains the bound at smaller $\alpha$ (Appendix~\ref{app:lower}). A threshold that uses all pairs does almost as well and is
dominated by COVE:

\begin{proposition}[a valid calibration-only threshold]\label{prop:inductive}
For $0\le\alpha<1$, let $S_{(k)}$ be the $k$-th largest calibration pair score and $k_\alpha=\lfloor\alpha N\rfloor-2m$. Under Assumption~\ref{ass:exch}, the rule ``accept iff $S_{e^\star}>S_{(k_\alpha)}$'' (reject all if $k_\alpha<1$) has $\FAR\le\alpha$; it can accept for some score arrays iff $\alpha\ge(2m+1)/N$, and $(2m+1)/N<4/(m+2)$. At $\alpha=1$, where $k_\alpha$ exceeds
$\binom m2$, the rule accepts every trial.
\end{proposition}

\section{COVE: a finite-sample transductive test}\label{sec:cove}

Theorem~\ref{thm:lower} leaves one way out: look at the trial's units themselves. For a trial with score $s^\star=S_{e^\star}$, let
\begin{align*}
R \;=\;& 1+\#\{e\subset[m]: S_e\ge s^\star\}\\
&+\textstyle\sum_{i\le m}\big(\ind\{S_{i,m+1}\ge s^\star\}+\ind\{S_{i,m+2}\ge s^\star\}\big)
\end{align*}
(the trial itself, calibration pairs, and cross pairs) be the rank of the trial among all $N$ pairs of the augmented set, ties counted against it, and $p=R/N$. COVE accepts iff $p\le\alpha$.

\begin{theorem}[finite-sample validity]\label{thm:valid}
Let $g$ be symmetric. Under Assumption~\ref{ass:exch}, for every $t\in[0,1]$, $\PR_{H_0}(p\le t)\le\lfloor tN\rfloor/N$, with equality if the $N$ pair scores are almost surely distinct. Hence $\FAR\le\lfloor\alpha N\rfloor/N\le\alpha$ for every law satisfying Assumption~\ref{ass:exch}.
\end{theorem}

\emph{Proof sketch.} Relabelling the units permutes the pairs without changing the law of the score array, and the permutations act transitively on pairs; the trial pair is therefore distributed like a uniformly random pair, and at most $k$ pairs of any array have rank $\le k$ (Appendix~\ref{app:valid}). \hfill$\square$

The guarantee is marginal, over the calibration draw and the trial, as in split conformal prediction. The rank argument is the one behind conformal inference for jointly exchangeable arrays \citep{marandon2024conformal,lunde2026dyadic}; what is specific here is its use as a verification test and its contrast with
Theorem~\ref{thm:lower}. Symmetry ensures validity for this unordered-pair construction; without it the guarantee can fail: with $g(u,v)=u-v$ and $m=1$ the trial is the top pair with probability $5/12>1/3$; a cohort normalization applied to one side only is of this kind. Three consequences follow.

\emph{Resolution.} COVE can first accept at $\alpha=1/N\approx2/m^2$ when its score is uniquely largest; ties can delay acceptance. This is the resolution the pair count suggests but calibration-only rules cannot guarantee.

\emph{Dominance.} If $S_{e^\star}>S_{(k_\alpha)}$, at most $k_\alpha-1$
calibration pairs and $2m$ cross pairs score $\ge s^\star$, so $R\le\lfloor\alpha N\rfloor$: COVE accepts every trial that the U-conformal rule of Proposition~\ref{prop:inductive} accepts, which proves that proposition.

\emph{Per-probe control.} For almost every probe value $u$, hub or not, the $m+1$ scores of $u$ against the calibration units and the reference are exchangeable, and COVE accepts only if the trial score is among the top $\lfloor\alpha N\rfloor$ of them. The FAR conditional on the probe is therefore at most $\lfloor\alpha N\rfloor/(m+1)\le(m+2)\alpha/2$
(Proposition~\ref{prop:probe}), whereas a calibration-only threshold can accept a hub probe's impostor trials with probability up to $\rho$. The bound is useful when $m\alpha$ is small, and its factor $(m+2)/2$ is the one in Theorem~\ref{thm:lower}(a).

\subsection{Algorithm, cost, and exact arithmetic}\label{sec:algo}

\begin{algorithm}[t]
\caption{COVE for one trial $(x,x')$}\label{alg:cove}
\begin{algorithmic}[1]
\REQUIRE quantized embeddings $U_1,\dots,U_m$ of $m$ calibration identities (one
image each); their pair scores sorted in decreasing order $S_{(1)}\ge\dots\ge S_{(\binom m2)}$;
a symmetric score $g$ (the same $g$ for all pairs); level $\alpha$
\STATE $k\gets\lfloor\alpha\binom{m+2}{2}\rfloor$; \textbf{if} $k=0$ \textbf{return} reject
\STATE $s^\star\gets g(x,x')$
\STATE $R\gets1+\#\{j:S_{(j)}\ge s^\star\}$ \hfill(binary search)
\STATE $R\gets R+\#\{i:g(x,U_i)\ge s^\star\}+\#\{i:g(x',U_i)\ge s^\star\}$
\STATE \textbf{return} accept if $R\le k$, else reject
\end{algorithmic}
\end{algorithm}

With a raw score the per-trial cost is $O(md+\log m)$; cohort-normalized scores
also need $O(|Z|d)$ work per new unit unless its cohort statistics are cached.
Calibration costs $O(m^2d+m^2\log m)$ once, plus cohort preprocessing when used. Theorem~\ref{thm:valid}
concerns a \emph{fixed} symmetric function, and floating-point matrix products are
not one: the same pair can round differently depending on summation order. We
therefore quantize every $\ell_2$-normalized embedding $f$ once to
$q=\mathrm{round}(\lambda f)$ with $\lambda=(4095-\sqrt d/2)/M$, where
$M=\max_i\|f_i\|_2$ is computed from the fixed feature cache before any calibration
draw, so that $\|q\|^2<2^{24}$; by
Cauchy--Schwarz every partial sum of a dot product is then an integer below $2^{24}$,
computed exactly in single precision by standard FP32 matrix-product kernels whatever the
summation order (a canary checks this before each run). The integer inner product
$q_u^\top q_v\approx\lambda^2\cos(f_u,f_v)$ is the score the guarantee is about
(Appendix~\ref{app:exact}).

\subsection{Cohort normalization with a guarantee}\label{sec:cohort}

Speaker and face verification normalize a trial score by the impostor score distribution of each side against a cohort \citep{auckenthaler2000score}, and
cross-modal retrieval rescales similarities to penalize hubs \citep{conneau2018word,bogolin2022qbnorm,pan2025dualbank}. These heuristics are
used because raw thresholds are unreliable; their cited formulations do not supply the distribution-free marginal FAR guarantee studied here. Let $Z$ be a
cohort of units from identities disjoint from calibration and deployment. Two symmetric scores that use it are
\begin{align*}
g_{\mathrm{csls}}(u,v) &= 2L\,g_0(u,v)-T_L(u)-T_L(v),\\
g_{\mathrm{snorm}}(u,v) &= \tfrac12\Big[\tfrac{g_0(u,v)-\mu_Z(u)}{\sigma_Z(u)}+\tfrac{g_0(u,v)-\mu_Z(v)}{\sigma_Z(v)}\Big],
\end{align*}
where $g_0$ is the raw integer score, $T_L(u)$ is the sum of $u$'s $L=10$ largest scores against $Z$ (cross-domain similarity local scaling, CSLS, in integer form), and
$\mu_Z(u),\sigma_Z(u)$ are the mean and standard deviation of all of $u$'s scores against $Z$ (symmetric normalization, S-norm).
Given $Z$, each is a fixed symmetric function of the pair, so Theorem~\ref{thm:valid} applies:

\begin{corollary}\label{cor:cohort}
If the calibration and trial units are exchangeable conditionally on $Z$ under $H_0$, COVE with $g_{\mathrm{csls}}$, $g_{\mathrm{snorm}}$, or any other symmetric cohort-based
score has $\FAR\le\lfloor\alpha N\rfloor/N$.
\end{corollary}

Normalization then changes \emph{which} trials are accepted while preserving the FAR upper bound, so it may be chosen for power using data independent of the calibration and trial units
(such as $Z$ itself). One-sided normalizations, and querybank schemes that normalize only the query, are not symmetric and are not covered.

\section{Power and the zoo index}\label{sec:power}

Let $\bar G(s)=\PR(S_{12}\ge s)$ be the impostor tail and $h(u;s)=\PR(g(u,U')\ge s)$ the impostor rate of a \emph{specific} unit $u$ against a unit of a random other identity. For a genuine trial $(u_a,u_b)$ with score
$s$, independent of i.i.d.\ calibration units,
\begin{align*}
\mu:=\frac{\E[R\mid\text{trial}]}{N}=\bar G(s)+\frac{1-\bar G(s)+m\,\omega(s)}{N},\\
\omega(s)=h(u_a;s)+h(u_b;s)-2\bar G(s).
\end{align*}
The oracle global threshold accepts iff $\bar G(s)\le\alpha$. COVE's expected rank has a finite-rank self-count correction $(1-\bar G(s))/N$ and an excess-rate term
$m\omega(s)/N\approx2\omega(s)/m$; its realised rank also varies with the calibration draw.

Calibration noise enters through the calibration exceedance fraction: for a fixed $\tau$, $\hat U(\tau)=\binom m2^{-1}\sum_{e\subset[m]}\ind\{S_e\ge\tau\}$ is a U-statistic of
order two.

\begin{proposition}[zoo index; a restatement of \citealp{hoeffding1948class}]\label{prop:zoo}
For i.i.d.\ units, with $\zeta_1=\mathrm{Var}\,h(U;\tau)$ and $\zeta_2=\bar G(\tau)(1-\bar G(\tau))$,
$\mathrm{Var}\,\hat U(\tau)=\big(4(m-2)\zeta_1+2\zeta_2\big)/\big(m(m-1)\big)$. With the \emph{zoo index} $\kappa=\zeta_1/\bar G(\tau)^2$, the
squared coefficient of variation of per-unit impostor rates, the relative standard error at $\bar G(\tau)=\alpha$ is
$\sqrt{4(m-2)\kappa/(m(m-1))+2(1-\alpha)/(m(m-1)\alpha)}$.
\end{proposition}

A binomial model for $\binom m2$ independent pairs predicts relative standard error $\sqrt{2(1-\alpha)/(m(m-1)\alpha)}$; the ratio is $\sqrt{1+2(m-2)\kappa\alpha/(1-\alpha)}$,
which \emph{grows} with the number of identities. Pairs are worth $m^2/2$ independent trials only while $\kappa\ll1/(m\alpha)$; otherwise about $m/(4\kappa\alpha)$. The
biometric ``zoo'' \citep{doddington1998sheep} is exactly heterogeneity in $h$, and that pair error rates are U-statistics is known \citep{conti2024assessing,fogliato2023confidence};
$\kappa$ turns this into a measurable effective-sample-size correction, the two-way analogue of the design effect for clustered calibration data \citep{noonan2026exceedance}.

\begin{proposition}[power]\label{prop:power}
Let calibration units be i.i.d., drawn from identities other than the trial's, and independent of the genuine trial. Let $k=\lfloor\alpha N\rfloor$ and
$v=2\,\mathrm{Var}\,\hat U(s)+4m\,(h(u_a;s)+h(u_b;s))/N^2$. For every $\eta>0$ and every event $B$ of genuine trials on which $\mu\le k/N-\eta\alpha$,
\[
\TAR_{\mathrm{COVE}}\ \ge\ \PR_{\mathrm{gen}}(B)-\E_{\mathrm{gen}}\big[\ind_B\min\{1,\,v/(\eta\alpha)^2\}\big].
\]
\end{proposition}

For fixed $\eta\in(0,1/2)$, if $\alpha m^2\ge8/\eta$, taking $B=\{\bar G(s)\le\alpha(1-2\eta),\ \Delta\le\eta m\alpha/4\}$, with $\Delta$ the positive part of the excess rate $\omega(s)$, gives
$\TAR_{\mathrm{COVE}}\ge\TAR_{\mathrm{oracle}}(\alpha(1-2\eta))-\PR_{\mathrm{gen}}(\Delta>\eta m\alpha/4)-\varepsilon$, where $\varepsilon\to0$ when $\alpha m^2\to\infty$ and $m^{-1}\sup_{\bar G(s)\le\alpha}\zeta_1(s)/\alpha^2\to0$
(Appendix~\ref{app:power}). Under the stated conditions, the bound attributes losses relative to the oracle threshold to large excess impostor rates at the operating point and calibration variability.
The bound is one-sided. COVE is not a threshold on $s$: it adapts to the calibration draw, so the best global threshold at the same marginal level is not an upper bound on its
power, and COVE can exceed it by accepting more in draws whose calibration identities are less hub-laden and fewer in the others (Q4 in Section~\ref{sec:experiments}).

\section{Experiments}\label{sec:experiments}

We ask: (Q1) is COVE's FAR at its level on real embeddings; (Q2) how much power does it keep where the valid calibration-only rules are trivial; (Q3) how do pair-based calibration recipes fare against their nominal targets; (Q4) what drives COVE's refusals; (Q5) does cohort normalization help; (Q6) what happens when Assumption~\ref{ass:exch} fails?

\paragraph{Setup.} We use six public benchmarks with identity-disjoint calibration and test sides ImageNet validation classes \citep{russakovsky2015imagenet}, Stanford Online Products (SOP, \citealp{song2016deep}), CUB-200-2011 \citep{wah2011cub}, Stanford Cars \citep{krause2013cars}, DeepFashion In-shop \citep{liu2016deepfashion} and a GLDv2 landmark subset \citep{weyand2020gldv2}, and eight frozen encoders: DINOv2-S/B/L \citep{oquab2024dinov2}, CLIP-B/16 and L/14 \citep{radford2021clip}, SigLIP-B/16 \citep{zhai2023siglip}, ResNet-50 \citep{wightman2021resnet} and ConvNeXt-B
\citep{liu2022convnext}. A \emph{cell} is one (dataset, encoder, score, split, $m$) configuration. The core tier (ImageNet, SOP, CUB, Cars $\times$ DINOv2-B, CLIP-B/16, ResNet-50; $m\le10^4$) uses \NReps\ calibration draws per cell; the new-dataset tier (In-shop, GLDv2), the new-encoder tier (five encoders on ImageNet and SOP) and the cohort,
split-shift and PAC tiers use \NRepsOther. Each draw samples $m$ identities and one image of each, and \NImpPerDraw\ impostor and \NGenPerDraw\ genuine trials from the remaining identities, identity-uniformly as Assumption~\ref{ass:exch} requires (Appendix~\ref{app:protocol}); levels range over $\alpha\in[10^{-7},10^{-1}]$. Baselines, named as in all tables and figures: two valid calibration-only rules, \emph{U-conformal} (Proposition~\ref{prop:inductive}) and \emph{disjoint-pairs conformal} (split conformal on $\lfloor m/2\rfloor$ disjoint pairs); \emph{disjoint-pairs PAC}, the same pairs with a Clopper-Pearson order statistic, PAC for i.i.d.\ identity draws; three recipes that treat the $\binom m2$ pairs as independent, the \emph{empirical quantile}, the \emph{binomial UCB} (Clopper-Pearson upper confidence bound) over pairs and \emph{conformal over pairs} (COVE
without its cross pairs); the 95\% \emph{identity bootstrap} \citep{bolle2004subsets,fogliato2023confidence} ($B=200$, $m\le1000$); and the \emph{oracle} global threshold, evaluated at COVE's exact level. All rules see the same draws and trials; the PAC tier, which adds the identity bootstrap, repeats the first \NRepsOther\ draws of core cells. The study has \NCells\ cells and took \EvalGPUHours\ GPU-hours (Appendix~\ref{app:compute}).

\begin{figure*}[t]
\centering
\includegraphics[width=\textwidth]{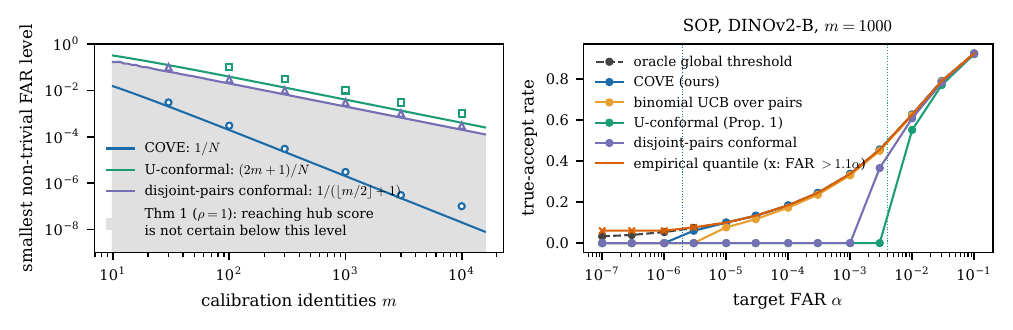}
\caption{\textbf{Identities, not pairs.} Left: smallest level at which each valid rule can accept for some score arrays (lines), and on SOP/DINOv2-B (markers: smallest grid level with non-zero TAR); below the shaded boundary Theorem~\ref{thm:lower} bounds the probability that a calibration-only threshold reaches the hub score when $\rho=1$. Right: true-accept rate against the target
level at $m=1000$; crosses mark levels at which the empirical quantile's realised FAR exceeds $1.1\alpha$; dotted lines mark $1/N$ and $(2m+1)/N$.}\label{fig:separation}
\end{figure*}

\paragraph{Q1: validity.} The random identity split samples calibration and trial identities from one finite pool on every draw, satisfying Assumption~\ref{ass:exch} by construction. Across \CoveValidityCells\ real-data (cell, level) comparisons with positive exact level, down to $\alpha=\CoveMinAlpha$, the median realised FAR divided by $\lfloor\alpha N\rfloor/N$ is
\CoveMedRatio, and \CoveNoneReject\ deviates significantly from its level (largest $|z|$ \CoveMaxAbsZ\ against the two-sided Bonferroni threshold \BonfZ; Figure~\ref{fig:evidence}a). This checks Theorem~\ref{thm:valid} end to end, quantization and sampling included; it does not certify a calibration-conditional FAR for a fixed deployment.

\paragraph{Q2: power where calibration-only rules give up.} The valid calibration-only rules cannot accept below $1/(\lfloor m/2\rfloor+1)$ and $(2m+1)/N$ (Figure~\ref{fig:separation}, left); COVE can accept from $1/N$ when its score is uniquely the largest. In the \NPowerCells\ (cell, level) combinations where both are trivial, COVE is not, and the sampled oracle is reliable (Appendix~\ref{app:protocol}), COVE's true-accept rate is a median \MainTARRatio\ of the sampled oracle's \OracleAtLevel\ (10th percentile \MainTARRatioQlo, minimum \CoveMinRatioPct\%; per dataset and encoder in Table~\ref{tab:breadth}, and one level per core pair in Table~\ref{tab:main}).

\begin{table*}[t]\centering\small\setlength{\tabcolsep}{4pt}
\caption{Core results, one level per core pair (random identity split; \NReps\ draws per cell,
the first \NRepsOther\ for the identity bootstrap); at each level both valid calibration-only rules
accept nothing. COVE: realised FAR over its level $\lfloor\alpha N\rfloor/N$, 95\% interval. TAR:
true-accept rates (oracle at COVE's level). Quantile: realised FAR$/\alpha$. Cond.\ FAR: \% of draws
whose conditional FAR exceeds $\alpha$ for the two rules claiming 95\% conditional control (nominal
5\%; ``--'': accepts nothing). Last column: \% of draws whose sampled FAR exceeds COVE's level
(COVE claims marginal control only).}\label{tab:main}
\begin{tabular}{llrrcccccccc}
\toprule
 & & & & COVE & \multicolumn{3}{c}{TAR} & quantile & \multicolumn{2}{c}{\% cond.\ FAR$>\alpha$} & \% FAR \\
\cmidrule(lr){6-8}\cmidrule(lr){10-11}
dataset & encoder & $m$ & $\alpha$ & FAR/level & COVE & oracle & id.\ boot. & FAR$/\alpha$ & bin.\ UCB & id.\ boot. & $>$level \\
\midrule
Cars & CLIP-B/16 & 60 & $10^{-3}$ & 1.03$\pm$0.08 & 0.074 & 0.080 & 0.085 & 1.3 & -- & 26 & 41 \\
Cars & DINOv2-B & 60 & $10^{-3}$ & 0.94$\pm$0.08 & 0.057 & 0.064 & 0.059 & 1.4 & -- & 20 & 32 \\
Cars & ResNet-50 & 60 & $10^{-3}$ & 0.99$\pm$0.09 & 0.015 & 0.016 & 0.017 & 1.4 & -- & 24 & 34 \\
CUB & CLIP-B/16 & 60 & $10^{-3}$ & 0.91$\pm$0.08 & 0.063 & 0.069 & 0.067 & 1.3 & -- & 21 & 34 \\
CUB & DINOv2-B & 60 & $10^{-3}$ & 0.97$\pm$0.09 & 0.167 & 0.178 & 0.169 & 1.2 & -- & 24 & 34 \\
CUB & ResNet-50 & 60 & $10^{-3}$ & 0.99$\pm$0.09 & 0.039 & 0.043 & 0.044 & 1.5 & -- & 24 & 36 \\
ImageNet & CLIP-B/16 & 300 & $10^{-4}$ & 1.00$\pm$0.05 & 0.072 & 0.074 & 0.048 & 1.2 & 3 & 10 & 42 \\
ImageNet & DINOv2-B & 300 & $10^{-4}$ & 1.00$\pm$0.05 & 0.270 & 0.286 & 0.166 & 1.1 & 1 & 4 & 43 \\
ImageNet & ResNet-50 & 300 & $10^{-4}$ & 0.98$\pm$0.05 & 0.293 & 0.310 & 0.198 & 1.1 & 2 & 4 & 40 \\
SOP & CLIP-B/16 & 1000 & $10^{-5}$ & 1.01$\pm$0.10 & 0.059 & 0.053 & 0.048 & 1.8 & 30 & 25 & 39 \\
SOP & DINOv2-B & 1000 & $10^{-5}$ & 0.96$\pm$0.06 & 0.100 & 0.102 & 0.071 & 1.1 & 7 & 5 & 45 \\
SOP & ResNet-50 & 1000 & $10^{-5}$ & 0.94$\pm$0.06 & 0.073 & 0.073 & 0.058 & 1.2 & 11 & 12 & 44 \\
\bottomrule
\end{tabular}
\end{table*}

\paragraph{Q3: pair-based calibration.} The empirical quantile has no distribution-free guaranty here; where $\alpha\binom m2\ge1$ its realized FAR exceeds $1.1\alpha$ in \NaiveFracOverResolved\% of cases and reaches \NaiveWorstRatioResolved$\,\alpha$. Conformal over pairs COVE without its cross pairs has a median \PairsIidTARRatioMed\ of COVE's true-accept
rate, but its realized FAR is significantly above $\alpha$ (one-sided $z>\BonfZ$) in \PairsIidRejectPct\% of comparisons, reaching \PairsIidFARMax$\,\alpha$: the $2m$ cross pairs buy validity at no power cost in the median. The binomial UCB claims calibration-conditional (PAC) control; its conditional FAR exceeds $\alpha$ in a median \NaiveCPMedSmallM\% of draws for
$m\le30$ but \NaiveCPMedLargeM\% for $m\ge3000$ (medians over cells and levels; worst \NaiveCPWorstPct\%; nominal 5\%), growing with $m$, consistent with Proposition~\ref{prop:zoo} (Figure~\ref{fig:evidence}b). The identity bootstrap is near its nominal 5\% in the median where the calibration set resolves the level ($\alpha\binom m2\ge10$: median \BootPACViolMedWell\%,
worst \BootPACViolMaxWell\%), exceeds it for $1\le\alpha\binom m2<10$ (median \BootPACViolMedMid\%), and fails where $\alpha\binom m2<1$ (median \BootPACViolMedUnder\%, FAR up to \BootFARRatioMaxUnder$\,\alpha$); where both are non-trivial it attains a median \BootTARvsCoveMed\ of COVE's true-accept rate. Disjoint-pairs PAC, exact for i.i.d.\ identity draws, stays near its target in our finite-pool draws (worst \DisjCPViolMax\% of draws) but is non-trivial in only \DisjCPNontrivCells\ of \DisjCPCombos\ (cell, level) combinations, where it attains a median \DisjCPTARvsCoveMed\ of COVE's true-accept rate. COVE claims only marginal control; its sampled FAR exceeds its level in a median \CoveDrawViolMed\% of draws (Section~\ref{sec:limitations}).

\begin{figure*}[t]\centering
\includegraphics[width=\textwidth]{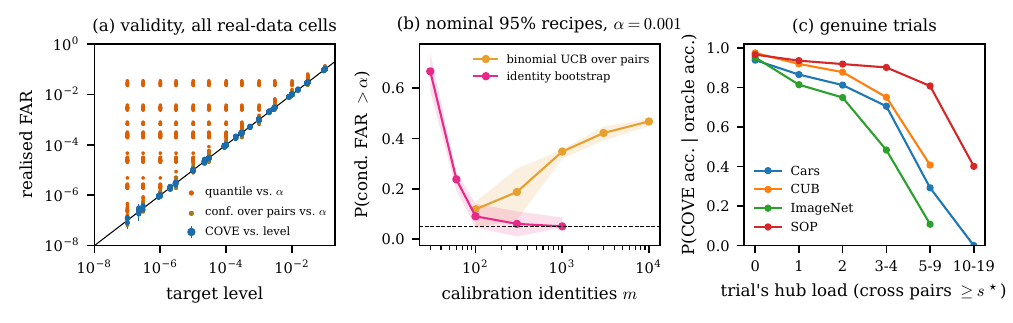}
\caption{\textbf{Evidence.} (a) Realised FAR against the target level for every random-split real-data cell and level (approximate 95\% intervals for the mean over draws): COVE against its level, the empirical quantile and the conformal p-value over pairs against $\alpha$; the hub model is in Appendix~\ref{app:breadth}. (b) Fraction of draws with estimated conditional FAR above $\alpha$
for nominal 95\% recipes, against $m$ (median and range over dataset/encoder pairs, $\alpha=10^{-3}$; nominal 5\% dashed). (c) Probability that COVE accepts a genuine trial the oracle accepts, by the trial's hub load, pooled over the encoders of each dataset; 50 draws, $m=60$, $\alpha=10^{-2}$ (CUB, Cars), $m=300$, $\alpha=3\cdot10^{-4}$ (ImageNet), $m=1000$,
$\alpha=10^{-4}$ (SOP).}\label{fig:evidence}
\end{figure*}

\paragraph{Q4: mechanism.}\label{sec:mechanism} Per-unit impostor rates are heterogeneous, increasingly so in the tail: over the eight dataset/encoder pairs we measured, the zoo index ranges from \KappaMinTwo--\KappaMaxTwo\ at $\alpha=10^{-2}$ to \KappaMinFour--\KappaMaxFour\ at $10^{-4}$, growing between these levels for every pair. Call the number of a trial's $2m$
cross pairs scoring at least $s^\star$ its \emph{hub load}. Among genuine trials the oracle accepts, COVE accepts \HubLoadZeroAcc\ of those with hub load 0 but \HubLoadThreeAcc\ of those with hub load at least 3 (Figure~\ref{fig:evidence}c): its refusal rate rises with hub load, consistent with Proposition~\ref{prop:power}. The sampled COVE true-accept rate exceeds the sampled oracle's in
\CoveBeatsOracleFrac\% of combinations (by up to \CoveBeatsOracleMaxPct\%); this is possible because COVE is not a global threshold on $s$ (Section~\ref{sec:power}). The largest directional cap occupancy over candidate directions (units, the mean direction, $k$-means centroids), a descriptive proxy for $\rho_Q$ rather than a bound, is \RhoMinFour--\RhoMaxFour\
($\approx$\RhoAlphaMinFour--\RhoAlphaMaxFour$\,\alpha$) at $\alpha=10^{-4}$.

\begin{table}[t]\centering\small\setlength{\tabcolsep}{2.5pt}
\caption{True-accept rate of COVE with the raw score and with two cohort normalizations
(Corollary~\ref{cor:cohort}), at the same guaranteed level, on the same \NRepsOther\
calibration draws and trials (random identity split).}\label{tab:cohort}
\begin{tabular}{llrrccc}
\toprule
data & encoder & $m$ & $\alpha$ & raw & CSLS & S-norm \\
\midrule
ImageNet & CLIP-B/16 & 300 & $10^{-4}$ & 0.073 & 0.092 & 0.095 \\
ImageNet & DINOv2-B & 300 & $10^{-4}$ & 0.272 & 0.257 & 0.254 \\
ImageNet & ResNet-50 & 300 & $10^{-4}$ & 0.294 & 0.313 & 0.299 \\
SOP & CLIP-B/16 & 1000 & $10^{-5}$ & 0.058 & 0.220 & 0.083 \\
SOP & CLIP-B/16 & 1000 & $10^{-4}$ & 0.128 & 0.402 & 0.190 \\
SOP & DINOv2-B & 1000 & $10^{-5}$ & 0.099 & 0.156 & 0.068 \\
SOP & DINOv2-B & 1000 & $10^{-4}$ & 0.184 & 0.378 & 0.209 \\
SOP & ResNet-50 & 1000 & $10^{-5}$ & 0.074 & 0.181 & 0.040 \\
SOP & ResNet-50 & 1000 & $10^{-4}$ & 0.138 & 0.356 & 0.116 \\
\bottomrule
\end{tabular}
\end{table}

\paragraph{Q5: cohort normalization.} Because normalization preserves COVE's FAR upper bound (Corollary~\ref{cor:cohort}), it can be judged on power alone. On the same draws and trials, CSLS with a 20\% held-out cohort raises COVE's true-accept rate by a median factor \CohortGainMedShown\ over the settings of Table~\ref{tab:cohort} (up to \CohortGainMax), but it
lowers it in \CslsWorseFrac\% of all cohort combinations (down to \CslsMinRatio$\times$), and S-norm lowers it in \SnormWorseFrac\%; the realised FAR stays at the level (Q1 includes these cells). Normalization should therefore be chosen on held-out data, which Corollary~\ref{cor:cohort} permits.

\paragraph{Q6: when identities are not exchangeable.}\label{sec:adverse} On SOP, calibrating on the official training products and testing on the official test products moves COVE's realised FAR to between \AdvFixedMin\ and \AdvFixedMax\ times its level; calibrating on six SOP super-categories and testing on the other six moves it to between \AdvCatMin\ and \AdvCatMax\
times (Table~\ref{tab:adverse}). The pair-based recipes also fail there: the empirical quantile's FAR reaches \AdvNaiveMax$\,\alpha$ and the binomial UCB violates its 95\% guarantee in up to \AdvCPMax\% of draws. On ImageNet the ``fixed'' split is one seeded random partition of classes, so its deviations show the spread induced by conditioning on a single partition rather than a population mismatch.

\paragraph{Breadth and cost.} On In-shop and GLDv2 and with five further encoders the conclusions are unchanged: validity as in Q1, and a median \TARRatioB\ (new datasets) and \TARRatioC\ (new encoders) of the oracle's true-accept rate (Appendix~\ref{app:breadth}). With the raw score a trial costs \UsPerTrialThousand\,$\mu$s at $m=1000$ and \UsPerTrialTenThousand\,$\mu$s at $m=10^4$ on
one GPU, amortized over batches of trials and including all baselines.

\section{Limitations}\label{sec:limitations}

\emph{Identity exchangeability.} Calibration and deployment identities must follow the same sampling protocol. The split-shift results (Q6) show that FAR control can fail otherwise. The assumption also fails if a probe identity occurs in calibration; Appendix~\ref{app:robust} proves a lower-power variant for a calibration match chosen independently of images, but we did not test it. The guarantee covers zero-effort impostors, not targeted look-alikes or presentation attacks.

\emph{Marginal control.} COVE controls FAR averaged over calibration draws, not the conditional FAR of a fixed draw. Proposition~\ref{prop:zoo} quantifies calibration tail-count variance at a fixed threshold; it does not bound COVE's
per-draw FAR. A conditional version remains open. Theorem~\ref{thm:lower}(b) also limits calibration-only PAC rules well below level $2\rho/m$.

\emph{Worst case and resolution.} The lower bound concerns a class of hub laws; its tightness for real encoders is unknown. The measured separation uses the achievable floors of the calibration-only baselines. Levels below $1/N\approx2/m^2$ are unreachable; $10^{-6}$ needs roughly $1.4\times10^3$ calibration identities.

\emph{Cost and scope.} Raw-score trials require $m$ stored-unit scores; cohort scores also require statistics against $Z$ for new units. We study one-to-one verification of objects, birds, cars, garments and landmarks with frozen encoders, not one-to-many identification, face or person verification. In-shop garment images show models; consent and fairness questions for
person verification lie outside our FAR guarantee.

\bibliography{references}

\subsection*{AI use statement}

We used generative AI tools to assist with language editing, readability improvements, and minor rephrasing of parts of the manuscript. All AI-assisted content was reviewed, edited, and verified by the authors. Generative AI tools were not used to formulate the research questions, develop the theoretical results, design the methodology, implement the proposed method, generate experimental data, or determine the scientific conclusions of this work. We take responsibility for the final content of this work, including all text, claims, analyses, and artifacts produced with the aid of generative AI tools.
We take responsibility for the final content of this work, including text, claims, or artifacts produced with the aid of generative AI.

% \subsubsection*{Acknowledgements}
% All acknowledgements go at the end of the paper, including thanks to reviewers who gave useful comments, to colleagues who contributed to the ideas, and to funding agencies and corporate sponsors that provided financial support.
% To preserve the anonymity, please include acknowledgements \emph{only} in the camera-ready papers. The acknowledgements do not count against the 9-page page limit in the camera-ready.

%%%%%%%%%%%%%%%%%%%%%%%%%%%%%%%%%%%%%%%%%%%%%%%%%%%%%%%%%%%%
\section*{Checklist}

\begin{enumerate}

\item For all models and algorithms presented, check if you include: 
\begin{enumerate} 
\item Yes. Sections 3-6 describe the mathematical setting, assumptions, decision rules, and the COVE algorithm. 
\item Yes. Section 5.1 discusses computational complexity and Section 4 characterizes sample-size and resolution limits. 
\item Yes. An anonymized implementation, dependencies, experimental scripts, and reproduction materials are included in the supplementary material. 
\end{enumerate}

\item For any theoretical claim, check if you include: 
\begin{enumerate} \item Yes. All theoretical results state their assumptions explicitly (e.g., Assumption 1 and theorem statements). \item Yes. Complete proofs are provided in Appendix A of the supplementary material. 
\item Yes. The assumptions and their practical implications are discussed in Sections 3, 5, 7, and 8. 
\end{enumerate}

\item For all figures and tables that present empirical results, check if you include: 
\begin{enumerate} 
\item Yes. Code, datasets, experimental protocols, audit procedures, and reproduction instructions are provided in the supplementary material. 
\item Yes. Appendix B describes datasets, encoders, splits, calibration protocols, hyperparameters, baselines, and evaluation procedures. 
\item Yes. Sections 7 and Appendix B define all reported metrics, confidence intervals, significance tests, and error bars. 
\item Yes. Computing infrastructure, GPU models, runtime, and memory usage are reported in Appendix E. 
\end{enumerate}

\item If you are using existing assets (e.g., code, data, models) or curating/releasing new assets, check if you include: \begin{enumerate} 
\item Yes. All datasets, pretrained models, and external assets are cited. 
\item Yes. License and usage information are summarized in Appendix D where available. 
\item Yes. The implementation and reproducibility artifacts are included in the supplementary material. 
\item Yes. Information provided by dataset creators and curators is summarized in Appendix D where applicable. 
\item Yes. Section 8 discusses limitations and ethical considerations. Dataset-specific considerations are documented in Appendix D. \end{enumerate}

\item If you used crowdsourcing or conducted research with human subjects, check if you include: 
\begin{enumerate} 
\item Not Applicable. 
\item Not Applicable. 
\item Not Applicable. 
\end{enumerate}
  % \item If you are using existing assets (e.g., code, data, models) or curating/releasing new assets, check if you include:
  % \begin{enumerate}
  %   \item Citations of the creator If your work uses existing assets. [Yes/No/Not Applicable]
  %   \item The license information of the assets, if applicable. [Yes/No/Not Applicable]
  %   \item New assets either in the supplemental material or as a URL, if applicable. [Yes/No/Not Applicable]
  %   \item Information about consent from data providers/curators. [Yes/No/Not Applicable]
  %   \item Discussion of sensible content if applicable, e.g., personally identifiable information or offensive content. [Yes/No/Not Applicable]
  % \end{enumerate}

  % \item If you used crowdsourcing or conducted research with human subjects, check if you include:
  % \begin{enumerate}
  %   \item The full text of instructions given to participants and screenshots. [Yes/No/Not Applicable]
  %   \item Descriptions of potential participant risks, with links to Institutional Review Board (IRB) approvals if applicable. [Yes/No/Not Applicable]
  %   \item The estimated hourly wage paid to participants and the total amount spent on participant compensation. [Yes/No/Not Applicable]
  % \end{enumerate}

\end{enumerate}

\clearpage
\appendix
\thispagestyle{empty}

% Supplementary material: To improve readability, you must use a single-column format for the supplementary material.
\onecolumn
\aistatstitle{Identities, Not Pairs: Supplementary Material}
\section{Proofs}\label{app:proofs}

Throughout, $\cP=\{\{i,j\}:1\le i<j\le m+2\}$, $N=|\cP|=\binom{m+2}{2}$, $S_e=g(U_i,U_j)$ for $e=\{i,j\}$, $e^\star=\{m+1,m+2\}$, and $R(e)=\#\{e'\in\cP:S_{e'}\ge S_e\}$. Where a result has a numerical counterpart in the released code, it is named in
brackets.

\subsection{Theorem~\ref{thm:valid} (finite-sample validity)}\label{app:valid}

\begin{lemma}\label{lem:count}
For every real array $(S_e)_{e\in\cP}$ and every integer $k\ge0$,
$\#\{e\in\cP:R(e)\le k\}\le k$.
\end{lemma}
\begin{proof}
Suppose the set $A=\{e:R(e)\le k\}$ has more than $k$ elements and let
$e_0\in A$ minimize $S_e$ over $A$. Every $e\in A$ satisfies $S_e\ge S_{e_0}$, so
$R(e_0)\ge|A|\ge k+1$, contradicting $e_0\in A$.
\end{proof}

\begin{proof}[Proof of Theorem~\ref{thm:valid}]
For a permutation $\sigma$ of $[m+2]$ let $\sigma(e)=\{\sigma(i),\sigma(j)\}$; this is a bijection of $\cP$. Define the permuted array $S^\sigma_e=S_{\sigma(e)}$, which is the array of scores of the relabelled units $(U_{\sigma(1)},\dots,U_{\sigma(m+2)})$
because $g$ is symmetric. Under $H_0$ and Assumption~\ref{ass:exch} the relabelled units have the same joint law as the original ones, hence $(S^\sigma_e)_{e\in\cP}\overset{d}{=}(S_e)_{e\in\cP}$. The rank of $e^\star$ in $S^\sigma$ is
$\#\{e':S_{\sigma(e')}\ge S_{\sigma(e^\star)}\}=\#\{e'':S_{e''}\ge S_{\sigma(e^\star)}\}=R(\sigma(e^\star))$.
Therefore $R(e^\star)\overset{d}{=}R(\sigma(e^\star))$ for every fixed $\sigma$. Now let $\sigma$ be uniformly distributed and independent of the units. Then $e_\sigma=\sigma(e^\star)$ is uniform on $\cP$ and independent of $S$, and for $k=\lfloor tN\rfloor$,
\[
\PR(p\le t)=\PR(R(e^\star)\le k)=\PR(R(e_\sigma)\le k)
=\E\Big[\tfrac1N\#\{e\in\cP:R(e)\le k\}\Big]\le \frac kN,
\]
by Lemma~\ref{lem:count}. If ties occur with probability zero, the $k$ pairs with the largest scores have ranks $1,\dots,k$ exactly, the inequality in Lemma~\ref{lem:count} is an equality, and $\PR(p\le t)=\lfloor tN\rfloor/N$.
\end{proof}

\emph{Remarks.} (i) Only invariance of the law of the score array under a group acting transitively on $\cP$ is used; pairs need not be independent, and the same proof covers ordered pairs for a non-symmetric $g$, with $N'=(m+2)(m+1)$. Symmetry is what makes the unordered array invariant: with $g(u,v)=u-v$, $m=1$ and i.i.d.\ Gaussian units, the trial pair has rank one with
probability $5/12>1/3$. (ii) $g$ may depend on data independent of all units (Corollary~\ref{cor:cohort}); the argument is then conditional on those data. (iii) If the probe's identity occurs in calibration, or calibration and deployment
identities come from different populations, Assumption~\ref{ass:exch} fails and so may the guarantee (Section~\ref{sec:adverse}). 

\subsection{Theorem~\ref{thm:lower} (calibration-only thresholds)}\label{app:lower}

\begin{proof}
\emph{Exchangeable part.} Let $H$ be a $\rho$-hub for $Q$ at level $\bar c$ and consider $\mathcal L_{Q,H}$: draw a position $J$ uniformly from $[m+2]$, let $U_J\sim H$ and $U_i\sim Q$ for $i\ne J$, all independent. This law is exchangeable.
On $\{J\in\{m+1,m+2\}\}$, which has probability $2/(m+2)$ and is independent of the calibration units, the calibration units are i.i.d.\ $Q$, so $\hat\tau$ (with its independent randomization) has its law under $Q$. The trial score is $g(V,U)$ with
$V\sim H$, $U\sim Q$ independent of calibration, and exceeds $\bar c$ with probability at least $\rho$. On $\{\hat\tau\le c\}\cap\{g(V,U)\ge\bar c\}$ the trial is accepted because $g(V,U)\ge\bar c\ge c\ge\hat\tau$. Therefore
\[
\alpha\ \ge\ \PR_{\mathcal L_{Q,H}}(\text{accept})\ \ge\ \frac{2}{m+2}\,\rho\,\PR_Q(\hat\tau\le c).
\]
Taking the supremum over admissible $H$ gives the claim with $\rho=\rho_Q(\bar c)$ (or any $\rho<\rho_Q(\bar c)$).

\emph{PAC part.} Let $\pi^*=1-\sqrt{1-\alpha/\rho}$, take any $\pi\in(\pi^*,1)$ and $Q_\pi=(1-\pi)Q+\pi H$, units i.i.d.\ $Q_\pi$, each carrying a latent label ``hub'' with probability $\pi$. Given the calibration units, the conditional FAR of any threshold
$t\le c$ is at least the probability that the trial contains a hub unit and scores $\ge\bar c$, which is at least $\rho\,(1-(1-\pi)^2)>\rho\,(1-(1-\pi^*)^2)=\alpha$ (a trial with one hub scores $\ge\bar c$ with probability $\ge\rho$ against the other unit, and a trial with two hubs with probability $\ge\rho$ by the hub-hub condition). Hence $\{\hat\tau\le c\}\subseteq\{\FAR(\hat\tau\mid U_{1:m},\xi)>\alpha\}$ and $\PR_{Q_\pi}(\hat\tau\le c)\le\delta$. On the event that no calibration unit carries the
hub label, of probability $(1-\pi)^m$, the calibration units are i.i.d.\ $Q$; so $(1-\pi)^m\PR_Q(\hat\tau\le c)\le\delta$. Letting $\pi\downarrow \pi^*$ gives $\PR_Q(\hat\tau\le c)\le\delta(1-\pi^*)^{-m}=\delta(1-\alpha/\rho)^{-m/2}$.
\end{proof}

\emph{Remarks.} (a) Latent labels avoid assuming that $Q$ and $H$ have disjoint supports, and a diffuse $H$ avoids two identical trial images; the hub-hub condition is used only in part (b). (b) The proof uses only that $\hat\tau$ is a measurable
function of the calibration units (not merely of their scores) and independent randomization; it therefore covers rules that inspect the embeddings. (c) With several images per calibration identity, the same proof applies when hubness is a property of
the identity, i.e.\ all images of a hub identity are drawn from $H$. (d) The i.i.d.\ analogue of part (a), contamination at rate $1/(m+1)$, gives $\PR_Q(\hat\tau\le c)\le e(m+1)^2\alpha/((2m+1)\rho)$. (e) For unquantized cosine similarity on the full sphere $S^{d-1}$ with $d\ge2$ and $\bar c<1$, where diffuse hub laws near every direction are admissible, $\sup_w Q(\cos(w,U)>\bar c)\le\rho_Q(\bar c)\le\sup_w Q(\cos(w,U)\ge\bar c)$. This sandwich need not hold if hub units are restricted to the frozen encoder's attainable image set. The upper bound follows by averaging cap occupancies over $V\sim H$. For the lower bound, choose a diffuse $H$ concentrated in an arbitrarily small neighbourhood of a direction $w$: hub-hub scores exceed $\bar c$, while the cross probability tends to at least the strict-cap mass by Fatou's lemma. Atoms on a cap boundary can make the two bounds differ, so an equality needs a boundary regularity condition. (f) For $\rho=1$ and even $m$ the bound of part (a) is attained when hub scores exceed all
other scores. With probability $q=(m+2)\alpha/2\le1$, independently of the units, accept iff the trial pair scores above all $m/2$ disjoint calibration pairs, and otherwise reject. The trial pair and the $m/2$ disjoint pairs are exchangeable under Assumption~\ref{ass:exch}, so $\FAR\le q/(m/2+1)=\alpha$ for every such law, while on hub-free calibration sets $\hat\tau\le c$ with probability $q$ for every $c<\bar c$ above all non-hub scores.

\subsection{COVE's FAR conditional on the probe}\label{app:probe}

\begin{proposition}\label{prop:probe}
Under Assumption~\ref{ass:exch}, for almost every $u$ under the law of $U_{m+1}$, $\PR(\text{COVE accepts}\mid U_{m+1}=u)\le\lfloor\alpha N\rfloor/(m+1)\le(m+2)\alpha/2$.
\end{proposition}
\begin{proof}
Conditionally on $U_{m+1}=u$, the units $U_1,\dots,U_m,U_{m+2}$ are exchangeable, hence so are the $m+1$ scores $g(u,U_j)$, $j\ne m+1$. Every one of these scores that is $\ge s^\star=g(u,U_{m+2})$ belongs to a distinct pair counted in $R$, so acceptance
($R\le k=\lfloor\alpha N\rfloor$) implies that $s^\star$ has rank at most $k$ among them. By Lemma~\ref{lem:count} applied to this row of $m+1$ exchangeable scores, the probability is at most $k/(m+1)$, and $k/(m+1)\le\alpha N/(m+1)=(m+2)\alpha/2$.
\end{proof}

\subsection{A variant robust to a probe identity present in calibration}\label{app:robust}

Deployments cannot always ensure that a probe's identity is absent from calibration. For each calibration unit $j$ let $c_j$ be the number of augmented pairs containing $j$ that score $\ge s^\star$, and set $p^+=(R-\min_jc_j)/\binom{m+1}2$.

\begin{proposition}\label{prop:robust}
Suppose that, under $H_0$, either Assumption~\ref{ass:exch} holds, or the probe's identity equals that of one calibration unit $J$, chosen independently of all images, and the units other than $U_J$ are exchangeable. Then $\PR(p^+\le t)\le t$.
\end{proposition}
\begin{proof}
For any $j$, $R-c_j$ is the rank of the trial among the $\binom{m+1}2$ pairs of the augmented set with unit $j$ removed, so $p^+\ge p_j:=(R-c_j)/\binom{m+1}2$ for every $j$. In the first case take $j=1$: the remaining $m+1$ units are exchangeable and
Theorem~\ref{thm:valid} applies to $p_1$. In the second case take $j=J$: removing the calibration unit that shares the probe's identity leaves $m+1$ exchangeable units of distinct identities, and Theorem~\ref{thm:valid} applies to $p_J$.
\end{proof}
The price is a resolution of $1/\binom{m+1}2$ and the removal of the least-crowded calibration unit; we did not evaluate this variant empirically.

\subsection{Proposition~\ref{prop:inductive} (a valid calibration-only threshold)}\label{app:inductive}

\begin{proof}
Let $k=k_\alpha\ge1$. If $S_{e^\star}>S_{(k)}$, every calibration pair with score $\ge S_{e^\star}$ is strictly above $S_{(k)}$, so there are at most $k-1$ of them; at most $2m$ cross pairs score $\ge S_{e^\star}$; and $e^\star$ itself is counted
once. Hence $R(e^\star)\le k+2m\le\lfloor\alpha N\rfloor$, i.e.\ the acceptance event is contained in $\{p\le\alpha\}$ (this is the dominance statement), and Theorem~\ref{thm:valid} gives $\FAR\le\lfloor\alpha N\rfloor/N\le\alpha$. The rule is
able to accept for some score arrays iff $k_\alpha\ge1$, i.e.\ $\alpha\ge(2m+1)/N$; ties may still force rejection.
\end{proof}

\subsection{Proposition~\ref{prop:power} (power)}\label{app:power}

\begin{proof}
Condition on a genuine trial $(u_a,u_b)$ with score $s$; calibration units are i.i.d., independent of it, and of other identities. Write
$R=1+\binom m2\hat U(s)+X_a+X_b$ with
$X_a=\#\{i\le m:g(u_a,U_i)\ge s\}\sim\mathrm{Bin}(m,h(u_a;s))$, and similarly $X_b$.
Then $\E[R\mid\text{trial}]=N\mu$ with $\mu$ as displayed in Section~\ref{sec:power}
(use $N=\binom m2+2m+1$). By $(x+y)^2\le2x^2+2y^2$ and $\binom m2\le N$,
$\mathrm{Var}(R\mid\text{trial})/N^2\le2\,\mathrm{Var}\,\hat U(s)+2\,\mathrm{Var}(X_a+X_b)/N^2\le v$,
since $\mathrm{Var}(X_a+X_b)\le2\,\mathrm{Var}X_a+2\,\mathrm{Var}X_b\le2m(h(u_a;s)+h(u_b;s))$.
On $B$, $\mu\le k/N-\eta\alpha$, so Chebyshev's inequality gives
$\PR(R>k\mid\text{trial})\le\PR(R-N\mu\ge N\eta\alpha\mid\text{trial})\le\min\{1,v/(\eta\alpha)^2\}$.
Hence $\TAR\ge\E_{\mathrm{gen}}[\ind_B(1-\min\{1,v/(\eta\alpha)^2\})]$.
\end{proof}

\emph{Choice of $B$.} Since $k/N\ge\alpha-1/N$ and
$\mu\le\bar G(s)+\big(1+m\Delta\big)/N$ with $\Delta=\omega(s)^+$,
the event $\{\bar G(s)\le\alpha(1-2\eta)\}\cap\{(2+m\Delta)/N\le\eta\alpha\}$ implies
$\mu\le k/N-\eta\alpha$. With $N\ge m^2/2$ the second condition holds when
$\alpha m^2\ge8/\eta$ and $\Delta\le\eta m\alpha/4$. On this event
$h(u_a;s)+h(u_b;s)\le2\alpha+\eta m\alpha/4$ and, by Proposition~\ref{prop:zoo},
$\mathrm{Var}\,\hat U(s)\le4\zeta_1(s)/m+2\alpha/(m(m-1))$, so
\[
\frac{v}{(\eta\alpha)^2}\le\frac{8\,\zeta_1(s)/\alpha^2}{\eta^2m}+\frac{4}{\eta^2m(m-1)\alpha}
+\frac{4m(2+\eta m/4)}{\eta^2N^2\alpha}.
\]
Each term vanishes when $\alpha m^2\to\infty$ and $m^{-1}\sup_{\bar G(s)\le\alpha}\zeta_1(s)/\alpha^2\to0$,
which gives the corollary in Section~\ref{sec:power}. The condition is on $\zeta_1(s)/\alpha^2=\kappa(s)\bar G(s)^2/\alpha^2$, not on $\kappa(s)$ alone, which can be large in the extreme tail.

\subsection{Proposition~\ref{prop:zoo} (zoo index)}\label{app:zoo}

For i.i.d.\ units, $\hat U(\tau)$ is a U-statistic with symmetric kernel $\phi(u,v)=\ind\{g(u,v)\ge\tau\}$; $h$ is a per-unit (per-image) rate, and under sampling of identities without replacement a finite-population correction applies. Its Hoeffding decomposition
\citep{hoeffding1948class} gives
$\mathrm{Var}\,\hat U=\binom m2^{-1}\big[2(m-2)\zeta_1+\zeta_2\big]$ with
$\zeta_1=\mathrm{Var}\,\E[\phi(U,U')\mid U]=\mathrm{Var}\,h(U;\tau)$ and
$\zeta_2=\mathrm{Var}\,\phi(U,U')=\bar G(\tau)(1-\bar G(\tau))$, which is the stated
formula. Dividing by $\bar G(\tau)^2$ at $\bar G(\tau)=\alpha$
gives the relative standard error stated in Proposition~\ref{prop:zoo}; the formula is exact for i.i.d.\ units. A
binomial model for $\binom m2$ independent pairs predicts relative standard error $\sqrt{2(1-\alpha)/(m(m-1)\alpha)}$; the ratio is $\sqrt{1+2\kappa(m-2)\alpha/(1-\alpha)}$.

\subsection{Exact integer similarity}\label{app:exact}

Embeddings are $\ell_2$-normalized, multiplied by a single scale per (dataset, encoder) computed from the fixed feature cache before calibration draws and chosen so that every rounded vector satisfies $\|q\|^2<2^{24}$, and rounded to integers. Specifically, the implementation uses $\lambda=(4095-\sqrt d/2)/M$ with $M=\max_i\|f_i\|_2$ over that cache. Every coordinate product and, by Cauchy--Schwarz, every partial sum of a dot product is an integer of magnitude below $2^{24}$, hence represented exactly in IEEE
single precision, whatever the summation order, blocking, or split of the matrix product, for standard FP32 kernels that form only partial sums of products (not Strassen-type or split-precision emulation, which the canary below would detect); TF32 and reduced-precision modes are disabled, because TF32 represents integers exactly only up to $2^{11}$. The pair score is therefore a fixed, exactly symmetric integer-valued function, as Theorem~\ref{thm:valid} requires. Before each run a canary compares the single-precision matrix product and the row-wise product with an int64 elementwise reference on production-sized blocks, off-diagonal entries included. With $q=\lambda f+\epsilon$, $\|\epsilon\|\le\sqrt d/2$, rounding changes a cosine similarity by at most $\sqrt d/\lambda+d/(4\lambda^2)\le1.2\cdot10^{-2}$ ($d\le2048$) and typically by about $10^{-4}$; ties among pair scores are rare and only make the test more conservative.

\section{Experimental protocol}\label{app:protocol}

\paragraph{Pools and identity splits.} For each dataset we form a pool of identities. A fixed random 20\% of the calibration-side identities is reserved as the external cohort $Z$ used only by the cohort-normalized scores (Section~\ref{sec:cohort}); these
identities never appear as calibration or trial units. In the main protocol (\emph{random identity split}), each calibration draw samples $m$ distinct identities uniformly from the remaining pool and one uniformly random image of each; trials use
only the identities not drawn. This realises Assumption~\ref{ass:exch} exactly. The \emph{fixed split} and \emph{category split} protocols of Q6 (Section~\ref{sec:experiments}) instead calibrate on one fixed set of identities (the official training products for SOP, a seeded 500/500 class partition for ImageNet, or six SOP super-categories) and test on
another.

\paragraph{Trials.} Impostor trials draw two distinct identities uniformly at random and one uniform image of each, exactly as calibration units are drawn; sampling images uniformly instead over-weights identities with many images and violates
Assumption~\ref{ass:exch} (we found and fixed this defect during development; the test \texttt{tests/test\_engine.py} fails on the defective sampler). Genuine trials draw a uniformly random identity with at least two images and two distinct images of it. Each
calibration draw uses \NImpPerDraw\ impostor and \NGenPerDraw\ genuine trials; all rules are evaluated on the same draw and trials, so comparisons are paired.

\paragraph{Levels, metrics and intervals.} Target levels are $\alpha\in\{10^{-1},3\cdot10^{-2},\dots,10^{-7}\}$. The \emph{realised FAR} of a rule in a cell is the mean over draws of the fraction of accepted impostor trials; its 95\% interval is $\pm1.96$ standard deviations of the per-draw fractions divided by $\sqrt{D}$, with $D$ the number of draws. For COVE we compare it with the exact level $\lfloor\alpha N\rfloor/N$ of Theorem~\ref{thm:valid}. For threshold rules we also compute the \emph{calibration-conditional} FAR of each draw, the fraction of a sample of $5\times10^7$ population impostor pairs scoring above the draw's threshold, and report
the fraction of draws in which it exceeds $\alpha$. The oracle threshold is the $(1-\alpha)$-quantile of the same population sample. For COVE's validity (Q1) the standardized deviation of a comparison is $z=(\text{realised FAR}-\lfloor\alpha N\rfloor/N)/\mathrm{se}$, and it is significant if $|z|$ exceeds the two-sided Bonferroni threshold for all comparisons at
family-wise level 5\%.

\paragraph{Baselines.} (i) \emph{Empirical quantile}: accept iff the trial score exceeds the $(\lfloor\alpha\binom m2\rfloor+1)$-th largest calibration pair. (ii) \emph{Binomial UCB} (over pairs): the largest threshold whose Clopper--Pearson
95\% upper bound on the FAR, treating the $\binom m2$ pairs as independent, is at most $\alpha$. (iii) \emph{Conformal over pairs}: the conformal p-value of the trial among calibration pairs, treated as exchangeable. (iv) \emph{Disjoint-pairs conformal}: split
conformal on $\lfloor m/2\rfloor$ disjoint calibration pairs (valid). (v) \emph{U-conformal}: Proposition~\ref{prop:inductive} (valid). (vi) \emph{Disjoint-pairs PAC}: on the same $n=\lfloor m/2\rfloor$ disjoint pairs, accept iff the trial exceeds the
$(k+1)$-th largest score, with $k$ the largest integer such that $\PR(\mathrm{Bin}(n,\alpha)\le k)\le\delta$ ($\delta=0.05$). When identities are i.i.d.\ draws these pairs are i.i.d.\ and the calibration-conditional FAR exceeds $\alpha$ with probability
at most $\delta$; our draws are without replacement from a finite pool, so this guarantee is approximate there. The rule is non-trivial only once $(1-\alpha)^n\le\delta$. (vii) \emph{Identity bootstrap}: resample the $m$ calibration identities with replacement $B=200$ times; in a resample with multiplicities $w_i$, pair $\{i,j\}$, $i\ne j$, has weight $w_iw_j$ (pairs within one identity are excluded); $t_b$ is the smallest calibration score whose weighted exceedance fraction is at most $\alpha$, and the threshold is the 95th percentile of $t_1,\dots,t_B$ (evaluated for $m\le1000$). (viii) \emph{Oracle}: the population threshold (not a procedure; reference only).

\paragraph{Oracle and conditional FAR.} The oracle global threshold at level $\ell$ is the $(1-\ell)$-quantile of $5\times10^7$ sampled population impostor scores, and its true-accept rate is measured on $2\times10^6$ sampled population genuine trials (identity-uniform, cohort identities excluded); it is a Monte-Carlo approximation. When we compare with COVE we evaluate
it at COVE's exact level $\lfloor\alpha N\rfloor/N$, and we report power ratios only where the sample holds at least 100 impostor pairs above the threshold (relative standard error of the oracle's FAR at most 10\%) and its realised FAR is within 10\% of the level (\OracleUnreliableDropped\ combinations are excluded for this reason), and where the oracle's true-accept rate is at least $0.02$. The calibration-conditional FAR of a threshold rule in a draw is the fraction of the sampled population impostor pairs above its threshold among pairs whose two identities both remain after that draw's calibration identities are removed.

\paragraph{Mechanism diagnostic.} To estimate per-unit impostor rates and the zoo index, we sample 5,000 probes and 40,000 references by selecting an identity uniformly and then one image uniformly within that identity, with replacement across draws. References sharing the probe's identity are excluded. The binomial sampling contribution to the variance of the estimated rates is
subtracted before reporting the zoo index. Directional cap occupancy is a separate descriptive, image-sampled calculation and is not interpreted as the theoretical hub parameter.

\paragraph{Hub model.} The synthetic experiment draws $2\times10^4$ identities with centres $\mathrm{normalize}(\beta e+z)$, $z\sim\mathcal N(0,I_d/d)$, $d=64$, where $e$ is a fixed direction and $\beta=0.6$, except for a fraction $p_{\mathrm{hub}}\in\{0,0.01,0.03\}$ of hub identities with $\beta=3$; each identity has three images $\mathrm{normalize}(\text{centre}+0.6\,\varepsilon)$,
$\varepsilon\sim\mathcal N(0,I_d/d)$. Draws use $2\times10^5$ impostor and $2\times10^4$ genuine trials, and the oracle uses $2\times10^7$ impostor pairs; everything else is as above.

\paragraph{Provenance and reproduction.} Raw result records identify their configuration, source commit, library versions and GPU. The audit (\path{tools/audit.py}) recomputes each summary from those records, checks that every reported record was written from a clean
source tree whose engine code matches the final version up to comments and docstrings, and checks each data-dependent qualitative claim of the text. A separate spot check re-runs one cell for each (tier, source commit) of the calibration-cell and hub-model tiers at the final code and compares their stored per-draw rows bit for bit (\path{tools/reproduce_check.py}, \path{audits/reproduce_check.json}); it does not establish bitwise reproduction of every run. The Git history itself is not part of the anonymous archive.

\section{Additional results}\label{app:breadth}

\begin{figure}[h]\centering
\includegraphics[width=0.9\textwidth]{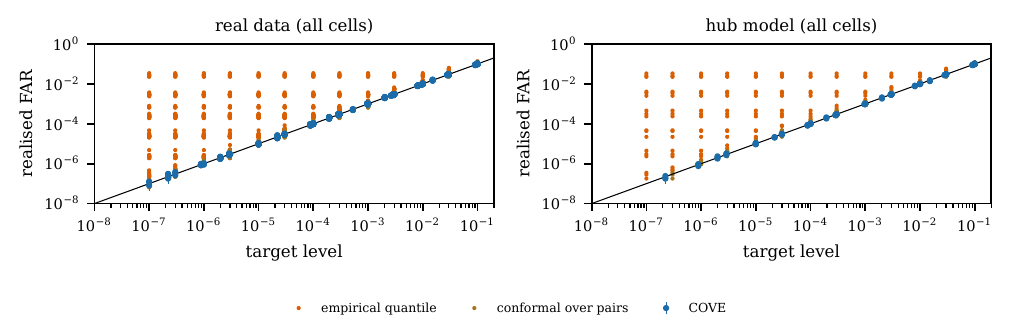}
\caption{Realised FAR against the target level for every random-split cell: real data (left) and
the hub model (right), with 95\% confidence intervals for the mean over calibration draws.}\label{fig:validity}
\end{figure}

\begin{table}[h]\centering\small
\caption{All datasets and encoders (raw score, random identity split). COVE's realised FAR
over its level (median and maximum over $m$ and $\alpha$), and COVE's true-accept rate as a
percentage of the oracle's where both valid calibration-only rules are trivial and the sampled
oracle is reliable (median, minimum, number of (cell, level) combinations).}\label{tab:breadth}
\begin{tabular}{llccccr}
\toprule
data & encoder & FAR/level med. & max & TAR \% med. & min & $n$ \\
\midrule
Cars & CLIP-B/16 & 1.01 & 1.03 & 99 & 90 & 9 \\
Cars & DINOv2-B & 0.99 & 1.02 & 99 & 89 & 9 \\
Cars & ResNet-50 & 0.99 & 1.03 & 100 & 92 & 8 \\
CUB & CLIP-B/16 & 1.00 & 1.04 & 100 & 90 & 9 \\
CUB & DINOv2-B & 1.00 & 1.08 & 97 & 89 & 9 \\
CUB & ResNet-50 & 1.01 & 1.06 & 102 & 91 & 9 \\
GLDv2 & CLIP-B/16 & 1.01 & 1.05 & 100 & 92 & 17 \\
GLDv2 & DINOv2-B & 1.00 & 1.03 & 99 & 89 & 17 \\
GLDv2 & ResNet-50 & 0.99 & 1.09 & 99 & 86 & 17 \\
ImageNet & CLIP-B/16 & 1.00 & 1.08 & 100 & 86 & 14 \\
ImageNet & CLIP-L/14 & 1.00 & 1.08 & 101 & 88 & 9 \\
ImageNet & ConvNeXt-B & 1.00 & 1.02 & 98 & 83 & 9 \\
ImageNet & DINOv2-B & 1.00 & 1.03 & 99 & 83 & 14 \\
ImageNet & DINOv2-L & 1.00 & 1.13 & 99 & 85 & 9 \\
ImageNet & DINOv2-S & 1.00 & 1.05 & 98 & 86 & 9 \\
ImageNet & ResNet-50 & 1.00 & 1.03 & 99 & 84 & 14 \\
ImageNet & SigLIP-B/16 & 1.00 & 1.02 & 100 & 87 & 9 \\
In-shop & CLIP-B/16 & 1.00 & 1.22 & 100 & 86 & 22 \\
In-shop & DINOv2-B & 1.00 & 1.40 & 100 & 88 & 19 \\
In-shop & ResNet-50 & 1.00 & 1.12 & 100 & 88 & 19 \\
SOP & CLIP-B/16 & 1.00 & 1.28 & 101 & 92 & 26 \\
SOP & CLIP-L/14 & 1.00 & 1.12 & 100 & 95 & 9 \\
SOP & ConvNeXt-B & 1.01 & 1.05 & 100 & 95 & 9 \\
SOP & DINOv2-B & 1.00 & 1.12 & 100 & 93 & 26 \\
SOP & DINOv2-L & 1.00 & 1.10 & 100 & 96 & 9 \\
SOP & DINOv2-S & 1.00 & 1.04 & 100 & 98 & 9 \\
SOP & ResNet-50 & 1.00 & 1.19 & 100 & 96 & 26 \\
SOP & SigLIP-B/16 & 1.00 & 1.10 & 99 & 94 & 9 \\
\bottomrule
\end{tabular}
\end{table}

\begin{table}[h]\centering\small
\caption{Realised FAR when calibration and test identities are not exchangeable: the official
fixed split, and calibration on six SOP super-categories with testing on the other six. COVE:
realised FAR over its level (median, range over $m$ and $\alpha$). Quantile: the empirical
quantile's realised FAR over $\alpha$ where $\alpha\binom m2\ge1$ (median, max). Binomial: median
percentage of draws whose calibration-conditional FAR exceeds $\alpha$ (binomial UCB).}\label{tab:adverse}
\begin{tabular}{lllcccccc}
\toprule
 & & & \multicolumn{2}{c}{COVE FAR/level} & \multicolumn{2}{c}{quantile FAR$/\alpha$} & binomial \\
split & data & encoder & median & range & median & max & \% draws \\
\midrule
category & SOP & CLIP-B/16 & 0.78 & 0.61--1.16 & 0.86 & 2.74 & 2 \\
category & SOP & DINOv2-B & 0.40 & 0.33--0.77 & 0.40 & 1.10 & 0 \\
category & SOP & ResNet-50 & 1.47 & 0.81--2.48 & 1.47 & 3.20 & 58 \\
fixed & ImageNet & CLIP-B/16 & 0.76 & 0.64--0.90 & 0.83 & 1.16 & 2 \\
fixed & ImageNet & DINOv2-B & 1.00 & 0.40--1.07 & 1.01 & 1.08 & 10 \\
fixed & ImageNet & ResNet-50 & 0.94 & 0.51--1.00 & 0.99 & 1.26 & 16 \\
fixed & SOP & CLIP-B/16 & 1.04 & 0.95--1.44 & 1.08 & 3.37 & 44 \\
fixed & SOP & DINOv2-B & 0.95 & 0.50--1.12 & 0.97 & 1.88 & 10 \\
fixed & SOP & ResNet-50 & 0.94 & 0.86--1.33 & 0.97 & 2.27 & 14 \\
\bottomrule
\end{tabular}
\end{table}

\section{Assets and licenses}\label{app:assets}

\begin{center}\small
\begin{tabular}{p{0.28\textwidth}p{0.34\textwidth}p{0.32\textwidth}}
\toprule
asset & source used & terms \\
\midrule
ImageNet-1k val \citep{russakovsky2015imagenet} & HF \path{ILSVRC/imagenet-1k} & ImageNet terms of access (non-commercial research)\\
SOP \citep{song2016deep} & HF \path{nyris/stanford-online-products-v1} & released by the authors for research; no explicit license found\\
CUB-200-2011 \citep{wah2011cub} & HF \path{Donghyun99/CUB-200-2011} & non-commercial research and educational use\\
Stanford Cars \citep{krause2013cars} & copy of the official release & released by the authors for research\\
DeepFashion In-shop \citep{liu2016deepfashion} & HF \path{Marqo/deepfashion-inshop} & non-commercial research only (DeepFashion agreement)\\
GLDv2 (mini subset) \citep{weyand2020gldv2} & HF \path{zguo0525/google-landmarks-v2-mini} & mirror per-image terms unverified\\
DINOv2 S/B/L \citep{oquab2024dinov2} & timm \path{*_dinov2.lvd142m} & Apache-2.0\\
CLIP B/16, L/14 \citep{radford2021clip} & timm \path{*_clip_224.openai} & Apache-2.0 (timm card); original MIT\\
SigLIP B/16 \citep{zhai2023siglip} & timm \path{vit_base_patch16_siglip_224.webli} & Apache-2.0\\
ResNet-50 \citep{wightman2021resnet} & timm \path{resnet50.a1_in1k} & Apache-2.0\\
ConvNeXt-B \citep{liu2022convnext} & timm \path{convnext_base.fb_in22k_ft_in1k} & Apache-2.0\\
\bottomrule
\end{tabular}
\end{center}

The original GLDv2 train set lists CC-BY images with per-image attribution metadata; the third-party mini mirror does not document individual image licenses, so we cannot verify its per-image terms. No source images are distributed in our supplement.

The anonymized supplementary manifest (\path{audits/feature_manifest.json}) records the model identifier, available dataset snapshot identifier, and SHA-256 of each cached feature array and archive. It does not record immutable checkpoint revisions for every
pretrained model, so exact feature extraction outside our feature cache is not established.

\section{Compute}\label{app:compute}

All runs used one GPU each, capped at 24\,GiB of memory so that they reproduce on a single consumer GPU; our machine had NVIDIA RTX 6000 Ada (48\,GB) and RTX 5000 Ada (32\,GB) cards. Feature extraction took \FeatureGPUHours\ GPU-hours in total and the
reported evaluations (calibration cells, mechanism diagnostics and oracle curves) \EvalGPUHours\ GPU-hours, with peak allocated memory \PeakMemGB\,GiB. Development runs not reported here pilots, runs superseded by defect fixes, and an earlier complete run
replaced by the clean rerun from which all reported numbers come took a further \DevGPUHours\ logged GPU-hours; \DevUntimedRecords\ development records carried no timing. No model was trained or fine-tuned.

\end{document}